\documentclass{article}

\usepackage{arxiv}

\usepackage[utf8]{inputenc} 
\usepackage[T1]{fontenc}    
\usepackage{hyperref}       
\usepackage{url}            
\usepackage{booktabs}       
\usepackage{amsfonts}       
\usepackage{nicefrac}       
\usepackage{microtype}      
\usepackage{lipsum}		
\usepackage{graphicx}
\usepackage{natbib}
\usepackage{doi}

\usepackage{subcaption}
\usepackage{tikz}
\usepackage{siunitx}
\usetikzlibrary{calc,angles,quotes,arrows}
\usepackage{relsize,exscale}

\usepackage[outdir=./]{epstopdf}
\usepackage{amsthm}
\usepackage{amsmath}
\usepackage{todonotes}
\usepackage{amssymb}
\usepackage{mathtools}

\graphicspath{ {images/} }

\newcommand*\diff{\mathop{}\!\mathrm{d}}

\newtheorem{lemma}{Lemma}
\newtheorem{definition}{Definition}
\newtheorem{theorem}{Theorem}

\title{Local Intrinsic Dimensionality Signals Adversarial Perturbations}

\author{{Sandamal Weerasinghe}\thanks{email: prameesha.weerasinghe@unimelb.edu.au}\\
	School of Computing and Information Systems\\
	University of Melbourne\\
	Australia \\
	\And
	{Tansu Alpcan} \\
	Department of Electrical and Electronic Engineering\\
	University of Melbourne\\
	Australia \\
	\And
	{Sarah M. Erfani} \\
	University of Melbourne\\
	Australia \\
	\And
	{Christopher Leckie} \\
	University of Melbourne\\
	Australia \\
	\And
	{Benjamin I. P. Rubinstein} \\
	University of Melbourne\\
	Australia \\
}

\begin{document}
\maketitle

\begin{abstract}
The vulnerability of machine learning models to adversarial perturbations
has motivated a significant amount of research under the broad umbrella
of adversarial machine learning. Sophisticated attacks may cause learning algorithms to learn decision functions or make decisions with poor predictive performance. In this context, there is a growing body of literature that uses local intrinsic dimensionality (LID), a local metric that describes the minimum number of latent variables required to describe each data point, for detecting adversarial samples and subsequently mitigating their effects. The research to date has tended to focus on using LID  as a practical defence method often without fully explaining why LID can detect adversarial samples. In this paper, we derive a lower-bound and an upper-bound for the LID value of a perturbed data point and demonstrate that the bounds, in particular the lower-bound, has a positive correlation with the magnitude of the perturbation. Hence, we demonstrate that data points that are perturbed by a large amount would have large LID values compared to unperturbed samples, thus justifying its use in the prior literature. Furthermore, our empirical validation demonstrates the validity of the bounds on benchmark datasets.
\end{abstract}


\section{Introduction}\label{intro}
Machine learning is fast becoming a key instrument in domains such as finance, cybersecurity, and engineering. However, there is a growing body of literature that suggests most machine learning algorithms, including deep neural networks (DNNs), are vulnerable to adversarial attacks. By adding carefully crafted perturbations to data points, adversaries aim to either avoid detection (i.e., test time attack) or alter the decision functions learned by the machine learning models (i.e., training data poisoning). As a consequence, applications that rely on machine learning for high-stakes automated decision making may take incorrect actions with severe consequences. Therefore, investigating defense mechanisms against adversarial attacks is a continuing concern within the adversarial machine learning community.

Adversarial attacks take place in situations where attackers have opportunities to alter existing data points or introduce new malicious data points. For example, training data poisoning can occur through the use of malware or when training data is collected using crowd-sourcing marketplaces, where organizations build data sets with the help of individuals whose authenticity cannot be guaranteed. Due to the size and complexity of datasets, it is infeasible to extensively validate the quality of all data/labels. Therefore, researchers have explored multiple approaches to address this problem and reduce the vulnerability of machine learning algorithms to adversarial attacks.

Recent research on adversarial machine learning defenses can be categorized as either a certified defense or an empirical defense. A certified defense attempts to learn provably robust prediction models against norm-bounded adversarial perturbations such as $\ell_2$ or $\ell_\infty$. For example, the recent works of \citet{wong2018provable}, \citet{cohen2019certified} and \citet{lecuyer2019certified} introduce certified defenses particularly for deep neural networks. In contrast, empirical defenses either incorporate adversarial data into the training process
\citep{pmlr-v119-zhang20z,madry2018towards,wang2019improving} or incorporate adversarial sample detection components to the machine learning pipeline. More recently, several works have explored empirical defenses that use Local Intrinsic Dimensionality (LID) \citep{LID1_Houle} as the basis for detecting adversarial samples. 

The LID value of a data point identifies the minimum number of latent variables required to represent that data point and recent evidence suggests that LID is an indicator of the degree of being an outlier \citep{houle2018correlation}. For example, \citet{pmlr-v80-ma18d} use LID to characterize adversarial subspaces in the context of Deep Neural Networks (DNNs). The authors use a novel learning strategy for DNNs based on LID that monitors the dimensionality of subspaces during training and adapts the loss function to minimize the impact of noisy labels. \citet{weerasinghe2020closing} introduce a black-box defense for nonlinear regression built around LID to identify and reduce the impact of adversarial samples.

Although previous works have established the intuition behind using LID for detecting adversarial samples, they fail to provide a theoretical justification for this behavior of LID. In this paper, we aim to address this by obtaining a lower-bound (and an upper-bound) for the LID value of a perturbed data point and show that there is an association between the size of the perturbation and the (lower-bound of the) LID value of the resulting data point.
Specifically, we theoretically analyze the impact of adversarial perturbations using a simple setup comprising three data points, (i) the benign data point being perturbed, $\mathbf{a}$, (ii) its position after being perturbed, $\mathbf{b}$, and (iii) a benign reference data point, $\mathbf{c}$. Using the distribution of distances from the perturbed location (i.e., $\mathbf{b}$), we present a lower bound and an upper bound for its LID value (\text{LID}($\mathbf{b}$)). Furthermore, by considering the two possible directions of perturbation (i.e., towards or farther from the reference point $\mathbf{c}$), we identify the least lower and upper bounds when perturbed towards $\mathbf{c}$ and the greatest lower and upper bounds when perturbed farther from $\mathbf{c}$. These theorems indicate a positive association between the LID value and the size of the perturbations, thereby justifying the use of LID as an indicator of adversarial perturbations in previous works.

The main contributions of this paper are summarized as follows. We present a lower bound and an upper bound for the LID value of a perturbed data point based on the distribution of distances to its neighbors. We then identify the least lower and upper bounds when the perturbation is towards the reference point, and the greatest upper and lower bounds when the perturbation is away from it. Finally, we provide empirical evidence on benchmark data sets to support our theoretical contributions.
%

\section{Literature Review}\label{sec:lit_review}
\subsection{Theory of Local Intrinsic Dimensionality}
The minimum number of latent variables needed to represent the data is considered as the intrinsic dimensionality (ID) of that particular data set. The ID of a dataset is thus an indicator of its complexity. A number of previous works have proposed characterizations of global intrinsic dimensionality (for data sets) and local intrinsic dimensionality (for a single data point) \citep{6406405}. More recently \citet{LID1_Houle} proposed a characterization of LID in which the distribution of distances to a query point is modeled in terms of a continuous random variable \citep{LID1_Houle, LID2_Houle}. \citet{amsaleg2015estimating} developed several estimators of LID based on extreme value theory,
using maximum likelihood estimation (MLE), the method of moments (MoM), probability weighted moments (PWM) and regularly varying functions (RV). Of these, the MLE model of LID estimation has been used for LID estimation in several works in adversarial machine learning including this paper (Definition \ref{def:2}). The accuracy of the MLE estimator is dependent on the number of neighbors considered in the calculation. More recently, \citet{amsaleg2019intrinsic} proposed an estimator of LID that achieves accurate results even with small neighborhood sizes.

Intrinsic dimensionality (through the use of estimators) has been utilized in many applications such as dimensionality reduction, outlier detection, similarity search and subspace clustering. In this paper, we focus on its ability to signal the degree of a sample being an outlier in the context of adversarial machine learning.

\subsection{Applications of LID to Adversarial Machine Learning}\label{sec:lit_lid_applications}
To date, several studies have investigated the usability of LID as a detector of adversarial samples. For example, \citet{LID_sarah_ICLR} use LID for detecting adversarial samples in Deep Neural Networks (DNNs). The authors use the LID values of the deep neural network representation (post-activation values at each layer) of images to train a classifier that identifies flipped labels with high success. For a similar application, \citet{pmlr-v80-ma18d} use a dimensionality-driven learning strategy which dynamically adapts the loss function of the DNN based on the dimensionality of subspaces (measured using LID). 

\citet{LID_SVM} introduce a novel defense mechanism against poisoning and label flipping attacks using Local Intrinsic Dimensionality (LID). The proposed defense employs several label dependent variants of LID (i.e., in-class LID, out-class LID and cross-class LID) against label flipping attacks. \citet{ijcai2021-437} proposed a black-box defense for nonlinear regression learning. First, the defense identifies samples that have the greatest influence on the learner's validation error and then uses LID to iteratively identify poisoned samples via a generative probabilistic model, and suppress their influence on the prediction function.

\citet{21375796} apply a LID measure called CrossLID to assess the degree to
which the distribution of data generated by a Generative Adversarial Network (GAN) conforms to the distribution of real data. The authors demonstrate that by incorporating the proposed LID based evaluation into the learning process, the generation quality of the GAN improves.

The work of \citet{amsaleg2020high} is the most closely related to the work presented in this paper. \citet{amsaleg2020high} introduced an upper-bound for the minimum perturbation required ($\delta$) to change the rank of a data point from $1$ to $k$ (or $k$ to $1$) w.r.t. a reference point in the context of an information retrieval application. Furthermore, the authors demonstrate that small perturbations are sufficient to change the ranks of data points with relatively larger LID values compared to data points with smaller LID values. 

In this work, we provide an upper-bound and a lower-bound for the LID value of a perturbed sample based on the distances (and the cumulative probability distribution of distances) to its starting location and a benign reference point. Furthermore, we identify the greatest (least) lower-bound and upper-bound when the data point is perturbed away from (towards) the reference point.

\section{Theoretical Overview of LID}\label{sec:background}
A particular challenge when using high dimensional data (e.g., images, natural language) is identifying lower dimensional representations with minimal loss of information. A considerable amount of literature has been published on characterizing the intrinsic dimensionality (i.e., the minimum number of latent variables required to describe the data) of data. These include fractal methods such as the Hausdorff dimension \citep{gneiting2012estimators} and Minkowski–Bouligand dimension \citep{falconer2004fractal}, topological approaches \citep{rozza2012novel}, and dimensional models such as minimum neighbor distance models \citep{rozza2012novel}, expansion dimension \citep{karger2002finding}, generalized expansion dimension \citep{6406405}, and LID \citep{6753958}. Of these, LID \citep{6753958} has been of particular interest to the machine learning community for detecting adversarial data points (Section \ref{sec:lit_lid_applications}). Therefore, in this paper, we focus on the adversarial detection capabilities of LID.

We start by discussing the theoretical background behind LID. First, we briefly introduce the theory of LID for assessing the dimensionality of data subspaces. Expansion models of dimensionality have previously been successfully employed in a wide range of applications such as manifold learning, dimension reduction, similarity search, and anomaly detection \citep{LID1_Houle,amsaleg2015estimating}. Next, we analyze the effect of perturbations on the LID values of data points. 

\subsection{Theory of Local Intrinsic Dimensionality.}\label{sec:lid_theory}
In the theory of intrinsic dimensionality, classical expansion models measure the rate of growth in the number of data samples encountered as the distance from the sample of interest increases \citep{LID1_Houle}. As an example, in Euclidean space, the volume of an \textit{m}-dimensional hyper-sphere grows proportionally to $r^m$, when its size is scaled by a factor of $r$. If we consider two hyper-spheres sharing the same center with volumes $V_{1}$ and $V_{2}$ and radii $r_1$ and $r_2$, respectively, the expansion dimension $m$ can be deduced as:
\begin{equation}
	\frac{V_{2}}{V_{1}}=
	\left(\frac{r_{2}}{r_{1}}\right)^m\Rightarrow m=\frac{\ln\left(V_{2}/V_{1}\right)}{\ln\left(r_2/r_1\right)}.\\
\end{equation}

Transferring the concept of expansion dimension to the statistical setting of continuous distance distributions leads to the formal definition of LID. By substituting the cumulative probability of distance for volume, LID provides measures of the intrinsic dimensionality of the underlying data subspace. Refer to the work of \cite{LID1_Houle} for more details concerning the theory of LID. The formal definition of LID is given below.
\begin{definition}[Local Intrinsic Dimensionality \citep{LID1_Houle}]\label{def:1}
	Given a data sample $\mathbf{a}\in\mathcal{S}$, let $R > 0$ be a random variable denoting the distance from $\mathbf{a}$ to other data samples. If the cumulative probability distribution function $F(r)$ of $R$ is positive and continuously differentiable at distance $r > 0$, the intrinsic dimensionality of $\mathbf{a}$ at distance $r$ is given by:
	\begin{equation}\label{eq:lid} 
		\text{ID}_{F}(r)\triangleq \lim_{\epsilon\to0^+}\frac{\text{ln}\big(F((1+\epsilon)\cdot r)/F(r)\big)}{\text{ln}(1+\epsilon)}=\frac{r\cdot F'(r)}{F(r)},
	\end{equation}
	whenever the limit exists.	
\end{definition}
The last equality of \eqref{eq:lid} follows by applying L'H\^{o}pital's rule to the limits \citep{LID1_Houle}. The local intrinsic dimension at $\mathbf{a}$ is in turn defined as the limit when the radius $r$ tends to zero:
\begin{equation}\label{eq:lid_2}
	\text{LID}(\mathbf{a})=\lim_{r\to0}\text{ID}_{F}(r).
\end{equation}

$\text{LID}(\mathbf{a})$ measures the probability density $F'(r)$ normalized by the cumulative probability density $F(r)/r$. In the ideal case where the data in the vicinity of $\mathbf{a}$ is distributed uniformly within a subspace, $\text{LID}(\mathbf{a})$ equals the dimension of the subspace; however, in practice these distributions are not ideal, the manifold model of data does not perfectly apply, and $\text{LID}(\mathbf{a})$ is not an integer \citep{LID_sarah_ICLR}. Nevertheless, the local intrinsic dimensionality is an indicator of the dimension of the subspace containing $\mathbf{a}$ that would best fit the data distribution in the vicinity of $\mathbf{a}$.

The smallest distances from point $\mathbf{a}$ can be regarded as extreme events associated with the lower tail of the underlying distance distribution. Using extreme
value theory (EVT), the following theorem shows that $F(r)$ completely determines $\text{ID}_F$.
\begin{theorem}[Local ID Representation \citep{houle2017local}]
	Let $F:\mathbb{R}^{\geq0}\rightarrow \mathbb{R}$ be a real-valued function, and assume that the $\text{ID}_{F}(0)$ exists. Let $r$ and $w$ be values for which $r/w$ and $F(r)/F(w)$ are both positive. If $F$ is non-zero and continuously differentiable everywhere in the interval containing $[\text{min}\{r, w\}, \text{max}\{r, w\}]$, then
	\begin{equation}\label{thm:theorem_1}
		\begin{aligned}
			\dfrac{F(r)}{F(w)}
			&=\Big(\dfrac{r}{w}\Big)^{\text{ID}_{F}(0)}\cdot G_{F,w}(r), \text{where}\\
			G_{F,w}(r)
			&\triangleq\exp\bigg(\mathop{\mathlarger{\int\limits_{r}^{w}}} \dfrac{\text{ID}_{F}(0)-\text{ID}_{F}(t)}{t}dt\bigg),
		\end{aligned}
	\end{equation}
	whenever the integral exists.
\end{theorem}

Furthermore, let $e>1$ be a constant. Then, as per Theorem 3 of \citet{houle2017local},
\begin{equation}
	\begin{aligned}
		\lim_{\substack{w\to0^{+} \\ 0<1/e\leq r/w\leq e}} G_{F,w}(r)=1.
	\end{aligned}
\end{equation}

The LID formula defined in \eqref{eq:lid_2} is the theoretical calculation of LID. We describe below the empirical estimation of LID ($\widehat{\text{LID}}$) as follows.
\begin{definition}[Estimation of LID \citep{amsaleg2015estimating}]\label{def:2}
	Given a reference sample $\mathbf{a} \sim \mathcal{P}$, where $\mathcal{P}$ represents the data distribution, the Maximum Likelihood Estimator of the LID at $\mathbf{a}$ is defined as follows:
	\begin{equation}\label{eq:lid_estimation}
		\widehat{\text{LID}}(\mathbf{a})=-\Bigg(\frac{1}{k}\sum_{i=1}^{k}\text{log}\frac{r_{i}(\mathbf{a})}{r_{\text{max}}(\mathbf{a})}\Bigg)^{-1}.
	\end{equation}
	Here, $r_{i}(\mathbf{a})$ denotes the distance between $\mathbf{a}$ and its $i$-th nearest neighbor within a sample of $k$ points drawn from $\mathcal{P}$, and $r_{\text{max}}(\mathbf{a})$ is the maximum of the neighbor distances.
\end{definition}
The above estimation assumes that samples are drawn from a tight neighborhood, in line with its development from extreme value theory. In practice, the sample set is drawn uniformly from the available training data (omitting $x$ itself), which itself is presumed to have been randomly drawn from $\mathcal{P}$. 

\section{Adversarial Perturbations and  LID}\label{sec:attack_model}

As stated in Section \ref{sec:lit_review}, multiple prior works use defenses built around LID for detecting and suppressing the impact of adversarial samples. In the literature, adversarial samples are usually generated by perturbing benign data in directions provided by the attack algorithm \citep{43405, moosavi2016deepfool,maini2020adversarial}. By perturbing the feature vectors, the adversary moves poisoned samples away from the domain of benign samples. Thus, each poisoned sample would have an irregular distribution of the local distance to its neighbors, which would be reflected by its LID value. For example, the LID estimate of the data point $\mathbf{a}\in \mathcal{S}$, i.e., $\widehat{\text{LID}}(\mathbf{a})$, is an indicator of the dimension of the subspace that contains $\mathbf{a}$, by comparing the LID estimate of a data sample to that of other data points, any data points that have substantially different lower-dimensional subspaces would get highlighted.


\begin{table}[h!]
	\begin{center}
		\caption{Summary of the Notation.}
		\label{tab:notation}
		\begin{tabular}{l|p{8cm}} 
			\hline
			$\mathcal{S}$ & The domain of data samples.\\
			$\mathbf{a}$ & The benign data point being perturbed ($\mathbf{a}\in\mathcal{S}$).\\
			$\mathbf{b}$ & The position of $\mathbf{a}$, after the perturbation ($\mathbf{b}\in\mathcal{S}$).\\
			$\mathbf{c}$ & A benign reference point ($\mathbf{c}\in\mathcal{S}$).\\
			$d(\mathbf{a},\mathbf{b})$ & A distance measure that measures the distance between any two data points in $\mathcal{S}$.\\
			$x$ & The distance between $\mathbf{a}$ and $\mathbf{c}$ measured using $d$.\\
			$y$ & The distance between $\mathbf{b}$ and $\mathbf{c}$ measured using $d$.\\
			$\delta$ & A real valued constant that determines the size of perturbations.\\
			$\delta x$ & The distance between $\mathbf{b}$ and $\mathbf{a}$ measured using $d$.\\
			$\mathcal{D}$ & The univariate distribution of the distances to other data points measured from any data point in $\mathcal{S}$.  \\
			$F(r)$ & The cumulative probability distribution function of $\mathcal{D}$.\\
			$\text{LID}(\mathbf{a})$ & The local intrinsic dimensionality of $\mathbf{a}$.\\
			$\widehat{\text{LID}}(\mathbf{a})$ & The numerical estimation of $\text{LID}(\mathbf{a})$.\\
			$n$ & The number of samples (data points) in $\mathcal{S}$ (i.e., $|\mathcal{S}|$).\\
			$p_{\mathbf{c}}$ & The cumulative probability of $F_{\mathbf{b}}$ at distance $y$ (i.e., $F_{\mathbf{b}}(y)$).\\
			$k_{\mathbf{c}}$ & The expected rank of $\mathbf{c}$ w.r.t the distribution of distances from $\mathbf{b}$ (i.e., $p_{\mathbf{c}}\times n)$.\\
			$\theta$ & The angle between $\mathbf{b}-\mathbf{a}$ and $\mathbf{c}-\mathbf{a}$.\\
		\end{tabular}
	\end{center}
\end{table}

\subsection{Problem Setup and Notation}

\begin{figure}[hp]
	\centering
	\tikzset{every picture/.style={line width=0.75pt}} 
	\begin{tikzpicture}[x=0.75pt,y=0.75pt,yscale=-1,xscale=1]
		
		\draw    (442.67,329.67) -- (544,330) ;
		\draw [shift={(544,330)}, rotate = 0.19] [color={rgb, 255:red, 0; green, 0; blue, 0 }  ][fill={rgb, 255:red, 0; green, 0; blue, 0 }  ][line width=0.75]      (0, 0) circle [x radius= 3.35, y radius= 3.35]   ;
		\draw    (400.5,288.4) -- (444,330) ;
		\draw [shift={(444,330)}, rotate = 43.72] [color={rgb, 255:red, 0; green, 0; blue, 0 }  ][fill={rgb, 255:red, 0; green, 0; blue, 0 }  ][line width=0.75]      (0, 0) circle [x radius= 3.35, y radius= 3.35]   ;
		\draw [shift={(398.33,286.33)}, rotate = 43.72] [fill={rgb, 255:red, 0; green, 0; blue, 0 }  ][line width=0.08]  [draw opacity=0] (8.93,-4.29) -- (0,0) -- (8.93,4.29) -- cycle    ;
		\draw  [dash pattern={on 0.84pt off 2.51pt}]  (397,286) -- (544,330) ;
		\draw  [draw opacity=0] (429.87,315.61) .. controls (432,315.03) and (434.45,314.72) .. (437.12,314.75) .. controls (448.05,314.86) and (458.72,320.58) .. (460.95,327.51) .. controls (461.2,328.3) and (461.34,329.07) .. (461.36,329.81) -- (441.16,327.3) -- cycle ; \draw   (429.87,315.61) .. controls (432,315.03) and (434.45,314.72) .. (437.12,314.75) .. controls (448.05,314.86) and (458.72,320.58) .. (460.95,327.51) .. controls (461.2,328.3) and (461.34,329.07) .. (461.36,329.81) ;
		\draw    (444,290) -- (444,350) ;
		
		\draw (419,335) node [anchor=north west][inner sep=0.75pt]    {$\mathbf{a}$};
		\draw (544,335) node [anchor=north west][inner sep=0.75pt]    {$\mathbf{c}$};
		\draw (379,279) node [anchor=north west][inner sep=0.75pt]    {$\mathbf{b}$};
		\draw (483,335) node [anchor=north west][inner sep=0.75pt]    {$x$};
		\draw (458,309) node [anchor=north west][inner sep=0.75pt]    {$\theta $};
		\draw (449,279) node [anchor=north west][inner sep=0.75pt]    {$y$};
		\draw (402,309) node [anchor=north west][inner sep=0.75pt]    {$\delta x$};
	\end{tikzpicture}
	\caption{The positions of the benign sample $\mathbf{a}$, its perturbed version $\mathbf{b}$, and reference point $\mathbf{c}$.}
	\label{fig:perturbations_figure}
\end{figure}
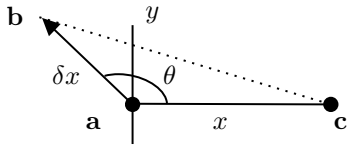

We start by describing the problem setup using three data points $\mathbf{a}, \mathbf{b}$ and $\mathbf{c}$ as depicted in Figure \ref{fig:perturbations_figure}. Take $\mathbf{a}\in \mathcal{S}$ as the initial (benign) point that the attacker aims to perturb and $\mathbf{b}\in \mathcal{S}$ as the resulting perturbed point. Take $\mathbf{c} \in \mathcal{S}$ as some benign reference location. In this section, we analyze how the perturbations affect $\text{LID}(\mathbf{b})$ (i.e., the LID value of the perturbed data point). Define $d(\mathbf{a},\mathbf{b})$ as the distance measure that measures the distance between any two data points in $\mathcal{S}$. Although we use Euclidean distance as the measure of distance in this paper, we note that any measure of distance for which the definition of LID holds is feasible. Thus, $\mathcal{D}_{\mathbf{b}}$ is defined as a univariate distribution of the distances to other data points measured from $\mathbf{b}$. 

Any point, for example $\mathbf{c}$, determines a distance $y=d(\mathbf{c}, \textbf{b})=\Vert \mathbf{b}-\mathbf{c}\Vert$, which in turn determines a probability $p_{\mathbf{c}}=F_{\mathbf{b}}(y)$. That is, any sample drawn from $\mathcal{S}$ has probability $p_{\mathbf{c}}$ being closer to $\mathbf{b}$. Therefore, with respect to the distribution $\mathcal{D}_{\mathbf{b}}$, the point $\mathbf{c}$ is given a distributional rank $p_{\mathbf{c}}$. We define the expected rank of $\mathbf{c}$ as $k_{\mathbf{c}}:=p_{\mathbf{c}}\times n$, where $n=|\mathcal{S}|$.

Lemma \ref{lem:1} considers a point $\mathbf{a}$ at distance $x$ from $\mathbf{c}$ (i.e., $x=\Vert \mathbf{a}-\mathbf{c}\Vert$). The point $\mathbf{b}$, which is produced by perturbing $\mathbf{a}$, has distance $\delta x$ from $\mathbf{a}$ for some proportion $\delta>0$ (see Fig. \ref{fig:perturbations_figure}). Define $\theta$ as the angle between $\mathbf{b}-\mathbf{a}$ and $\mathbf{c}-\mathbf{a}$. The lemma gives sufficient conditions on $\delta$ and $\theta$ for $\mathbf{b}$ to be closer to (resp. farther from) $\mathbf{c}$ compared to $\mathbf{a}$.

\begin{lemma}\label{lem:1}
	Define the distance $y=\Vert\mathbf{b}-\mathbf{c}\Vert$ and the associated cumulative probability $p_{c}=F_{\mathbf{b}}(y)$. Similarly, 
	define the distance $\delta x=\Vert\mathbf{b}-\mathbf{a}\Vert$ and the associated cumulative probability $p_{\mathbf{a}}=F_{\mathbf{b}}(\delta x)$ where $\delta>0$ and $x=\Vert\mathbf{a}-\mathbf{c}\Vert$. For $y>x$, that is when $\mathbf{b}$ is perturbed farther from $\mathbf{c}$ (relative to $\mathbf{a}$), $\delta > 2\cos\theta$. For $y<x$, (i.e., perturbed towards $\mathbf{c}$) $\delta < 2\cos\theta$ and $-\frac{\pi}{2}<\theta<\frac{\pi}{2}$. 
\end{lemma}
\begin{proof}
	Since $\theta$ is the angle between $\mathbf{b}-\mathbf{a}$ and $\mathbf{c}-\mathbf{a}$, by construction, we have:
	\begin{equation}\label{eq:y_distance}
		\begin{aligned}
			\Vert\mathbf{b}-\mathbf{c}\Vert^2
			& = \Vert\mathbf{b}-\mathbf{a}+\mathbf{a}-\mathbf{c}\Vert^2,                                                  \\
			& = \Vert\mathbf{b}-\mathbf{a}\Vert^2 + \Vert\mathbf{c}-\mathbf{a}\Vert^2 - 2(\mathbf{b}-\mathbf{a}).(\mathbf{c}-\mathbf{a}), \\
			& = \delta^2x^2 + x^2 - 2\delta x^2\cos\theta.                                                
		\end{aligned}
	\end{equation}
	In the first case where $y>x$ ($\mathbf{b}$ is perturbed farther from $\mathbf{c}$) we have,
	\begin{equation}
		\begin{aligned}
			\Vert\mathbf{b}-\mathbf{c}\Vert^2
			& > \Vert\mathbf{a}-\mathbf{c}\Vert^2,\\
			\delta^2x^2 + x^2 - 2\delta x^2\cos\theta & > x^2,\\
			\delta x^2(\delta - 2\cos\theta)& > 0.
		\end{aligned}
	\end{equation}
	Since $\delta>0$ and $x>0$, we have 
	\begin{equation}\label{eq:xyz_distances_farther}
		\begin{aligned}
			\delta & > 2\cos\theta. 
		\end{aligned}
	\end{equation}
	Following a similar approach, for the case where $y<x$ ($\mathbf{b}$ is perturbed towards $\mathbf{c}$) we obtain:
	\begin{equation}\label{eq:xyz_distances_towards}
		\begin{aligned}
			\delta& < 2\cos\theta~\text{and}~-\frac{\pi}{2}<\theta<\frac{\pi}{2}. 
		\end{aligned}
	\end{equation}
\end{proof}

\section{Theoretical Bounds on LID}
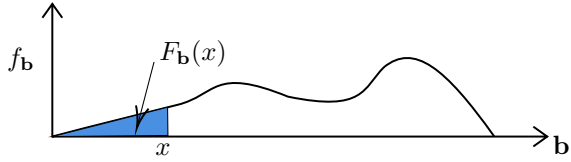
\begin{figure}[t]
	\centering
	\begin{tikzpicture}[x=0.75pt,y=0.75pt,yscale=-1,xscale=1]
		\draw [color={rgb, 255:red, 0; green, 0; blue, 0 }  ,draw opacity=1 ][line width=0.75]  (29.2,183.16) -- (278.8,183.16)(29.2,116.2) -- (29.2,183.6) (271.8,178.16) -- (278.8,183.16) -- (271.8,188.16) (24.2,123.2) -- (29.2,116.2) -- (34.2,123.2)  ;
		\draw [line width=0.75]    (29.2,183.16) .. controls (49.88,177.63) and (76.63,171.13) .. (92.38,167.13) .. controls (111.98,162.32) and (109.91,147.52) .. (148,163) .. controls (186.8,171) and (186,158.6) .. (196.4,148.2) .. controls (211.6,135.8) and (227.2,148.2) .. (252,183) ;
		\draw  [fill={rgb, 255:red, 74; green, 144; blue, 226 }  ,fill opacity=1 ] (29.2,183.16) -- (87.32,168.37) -- (87.5,182.61) -- cycle ;
		\draw    (80.5,145.5) -- (72.34,176.9) ;
		\draw [shift={(71.83,178.83)}, rotate = 284.57] [color={rgb, 255:red, 0; green, 0; blue, 0 }  ][line width=0.75]    (10.93,-3.29) .. controls (6.95,-1.4) and (3.31,-0.3) .. (0,0) .. controls (3.31,0.3) and (6.95,1.4) .. (10.93,3.29)   ;
		
		\draw (82,133.2) node [anchor=north west][inner sep=0.75pt]    {$F_{\mathbf{b}}(x)$};
		\draw (80,185.4) node [anchor=north west][inner sep=0.75pt]    {$x$};
		\draw (280.4,181.4) node [anchor=north west][inner sep=0.75pt]    {$\mathbf{b}$};
		\draw (5.2,137) node [anchor=north west][inner sep=0.75pt]    {$f_{\mathbf{b}}$};
	\end{tikzpicture}
	\caption{The probability distribution of distances from $\mathbf{b}\in \mathcal{S}$ induced by $\mathcal{D}_\mathbf{b}$.}
	\label{fig:distance distribution}
\end{figure}

We now consider the relationship between local intrinsic dimensionality and the effect of perturbations on neighborhoods under appropriate smoothness conditions of the underlying distance distribution. First, we define the conditions under which the LID is continuous at $\mathbf{b}\in \mathcal{S}$ \citep{amsaleg2020high}:
\begin{enumerate}
	\item There exists a distance $y>0$ for which all points $\mathbf{c}\in\mathbb{R}^d$ with $\Vert \mathbf{c}-\mathbf{b}\Vert\leq y$ admit a distance distribution $\mathcal{D}_{\mathbf{b}}$ whose cumulative distribution function $F_{\mathbf{b}}$ is continuously differentiable and positive within some open interval with a lower bound of $0$.
	\item For any sequence $\mathbf{b}\rightarrow\mathbf{c}$ of points satisfying Condition 1, there is convergence in distribution of the sequence of random distance variables defined at $\mathbf{b}$ to the distance variable defined at $\mathbf{c}$; that is, the condition $\lim_{\mathbf{b}\to\mathbf{c}} F_{\mathbf{b}}(\epsilon)=F_{\mathbf{c}}(\epsilon)$ holds for any distance $\epsilon\in(0,y)$.
	\item For each $\mathbf{b}$ satisfying Condition 1, $\text{LID}(\mathbf{b})$ exists and is positive.
	\item $\lim_{\mathbf{b}\to\mathbf{c}} \text{LID}(\mathbf{b})=\text{LID}(\mathbf{c})$.
\end{enumerate}
For the remainder of this Section, we assume that the LID is continuous at point $\mathbf{b}\in \mathcal{S}$. In practice, we find that these technical conditions are often satisfied as illustrated by the experimental results in Section \ref{sec:results}.

The following theorem derives an upper-bound and a lower bound for the LID of a perturbed data point, $\mathbf{b}$, based on the distances to its original location, $\mathbf{a}$, and a second reference data point $\mathbf{c}$, which are shown in 
Fig. \ref{fig:perturbations_figure}.
\begin{theorem}\label{thm:main_theorem}
	Define $y=\Vert\mathbf{b}-\mathbf{c}\Vert$, $x=\Vert\mathbf{a}-\mathbf{c}\Vert$ and $\delta x=\Vert\mathbf{b}-\mathbf{a}\Vert$. Let $F_{\mathbf{b}}(\delta x)$ be the cumulative probability of the distance distribution $\mathcal{D}_{\mathbf{b}}$ at distance $\delta x$. Similarly, let $F_{\mathbf{b}}(y)$ be the cumulative probability of the distance distribution $\mathcal{D}_{\mathbf{b}}$ at distance $y$. Thus, $p_{\mathbf{a}}=F_{\mathbf{b}}(\delta x)$ and $p_{\mathbf{c}}=F_{\mathbf{b}}(y)$ are the distributional ranks of $\mathbf{a}$ and $\mathbf{c}$ w.r.t the distribution of distances from $\mathbf{b}\in\mathcal{S}$. Similarly, define $k_{\mathbf{a}}=np_{\mathbf{a}}$ and $k_{\mathbf{c}}=np_{\mathbf{c}}$ as the corresponding expected ranks. Let $\delta$ be a real constant that is bounded as $0<\delta<y/x$. Assume $\text{LID}(\mathbf{b})$ exists and is positive. Let $\eta$ be a sufficiently small real value that is bounded as $0<\eta<\min\big(\frac{y}{\delta x}-1,\text{LID}(\mathbf{b})\ln(\varphi)/\vert \ln(\frac{\delta x}{y})\vert\big)$, where $\varphi=\text{min}\{\frac{y+\delta x\eta}{y},\frac{y}{y-\delta x\eta}\}$. For feasible values of $\eta$, there exists a positive integer $n_{0}>\max\{k_{\mathbf{a}}, k_{\mathbf{c}}\}$ for which the following inequalities hold for all choices of integer $n\geq n_{0}$,
	\begin{equation}\label{eq:lid_y_inequality}
		\begin{aligned}
			\dfrac{\ln\big(F_{\mathbf{b}}(y)\big/F_{\mathbf{b}}(\delta x)\big)}{\ln(\frac{y}{\delta x}+\eta)}\leq \text{LID}(\mathbf{b})\leq\dfrac{\ln\big(F_{\mathbf{b}}(y)\big/F_{\mathbf{b}}(\delta x)\big)}{\ln(\frac{y}{\delta x}-\eta)}. 
		\end{aligned}
	\end{equation}
\end{theorem}

\begin{proof}
	For a given choice of $n$, with $\mathbf{a}, \mathbf{b}, \mathbf{c}, x$ and $\delta$ as defined above, we have $p_{\mathbf{a}}=k_{\mathbf{a}}/n$ and $p_{\mathbf{c}}=k_{\mathbf{c}}/n$. Using the local ID characterization formula of the representation theorem (i.e., Theorem \ref{thm:theorem_1}), we observe that
	\begin{equation}
		\dfrac{k_{\mathbf{c}}}{k_{\mathbf{a}}}=\dfrac{p_{\mathbf{c}}}{p_{\mathbf{a}}}=
		\dfrac{F_{\mathbf{b}}(y)}{F_{\mathbf{b}}(\delta x)}=\bigg(\dfrac{y}{\delta x}\bigg)^{\text{LID}(\mathbf{b})}\cdot G_{F_{\mathbf{b},\delta x}}(y),
	\end{equation}
	where $G_{F_{\mathbf{b},\delta x}}(r)\triangleq\exp\bigg(\mathop{\mathlarger{\int\limits_{y}^{\delta x}}} \dfrac{\text{LID}(\mathbf{b})-\text{ID}_{F_{\mathbf{b}}}(u)}{u}du\bigg)$. 
	Upon rearranging, we obtain 
	\begin{equation}
		\dfrac{y}{\delta x}=\bigg(\dfrac{F_{\mathbf{b}}(y)}{F_{\mathbf{b}}(\delta x)\cdot G_{F_{\mathbf{b},\delta x}}(y)}\bigg)^{1/\text{LID}(\mathbf{b})}.
	\end{equation}
	Following a logarithmic transformation and substituting for $G_{F_{\mathbf{b},\delta x}}(y)$, we have
	\begin{equation}\label{eq:integral}
		\ln\Big(\frac{y}{\delta x}\Big)=
		\frac{1}{\text{LID}(\mathbf{b})}\ln\Big(\frac{F_{\mathbf{b}}(y)}{F_{\mathbf{b}}(\delta x)}\Big)-
		\frac{1}{\text{LID}(\mathbf{b})}\mathop{\mathlarger{\int\limits_{y}^{\delta x}}} \dfrac{\text{LID}(\mathbf{b})-\text{ID}_{F_{\mathbf{b}}}(u)}{u}\diff u.
	\end{equation}
	It is assumed that $F_{\mathbf{b}}(y)\neq F_{\mathbf{\delta x}}(y)$, which implies $y\neq \delta x$. Therefore, as per \citet{houle2017local}, for some real value $\eta$ such that 
	$0<\eta<\frac{y}{\delta x}-1$, there must exist a sufficiently small distance value $0<w$ such that any distance $u$ where $0<u<w$, implies $\vert\text{LID}(\mathbf{b})-\text{ID}_{F_{\mathbf{b}}(u)}\vert<\eta$. Therefore, the integral in \eqref{eq:integral} is bounded above by
	\begin{equation}
		\Bigg\vert\mathop{\mathlarger{\int\limits_{y}^{\delta x}}} \dfrac{\text{LID}(\mathbf{b})-\text{ID}_{F_{\mathbf{b}}}(u)}{u}\diff u\Bigg\vert\leq
		\eta\cdot\Bigg\vert\mathop{\mathlarger{\int\limits_{y}^{\delta x}}}\frac{1}{u}\diff u\Bigg\vert.
	\end{equation}
	Furthermore, since $\eta$ is defined as a small positive value, we bound $\eta$ from above as
	\begin{equation}
		\eta\leq
		\dfrac{\text{LID}(\mathbf{b})\ln(\varphi)}{\vert \ln(\frac{\delta x}{y})\vert},
	\end{equation}
	where $\varphi$ is the real valued variable
	\begin{equation}
		\varphi=\text{min}\Bigg\{\frac{y+\delta x\eta}{y},\frac{y}{y-\delta x\eta}\Bigg\}>1.
	\end{equation}	
	Define $n_{0}$ as the minimum size of the data set for which $F_{b}$ is strictly positive and continuously differentiable within some open interval of distances with lower endpoint $0$. For choices of $\eta$ as defined above, whenever the size of the data set $n$ is beyond the minimum ($n\geq n_{0}$), we have,
	\begin{equation}\label{eq:main_vert_ineq}
		\begin{aligned}
			\Bigg\vert\ln\big(\frac{y}{\delta x}\big)-\frac{1}{\text{LID}(\mathbf{b})}\ln\Big(\frac{F_{\mathbf{b}}(y)}{F_{\mathbf{b}}(\delta x)}\Big)\Bigg\vert
			\leq & \frac{1}{\text{LID}(\mathbf{b})}\Bigg\vert\mathop{\mathlarger{\int\limits_{y}^{\delta x}}} \dfrac{\text{LID}(\mathbf{b})-\text{ID}_{F_{\mathbf{b}}}(u)}{u}\diff u\Bigg\vert,\\
			\leq & \frac{\eta}{\text{LID}(\mathbf{b})}\Bigg\vert\mathop{\mathlarger{\int\limits_{y}^{\delta x}}}\frac{1}{u}\diff u\Bigg\vert,\\
			= & \frac{\eta}{\text{LID}(\mathbf{b})}\Big\vert\ln\Big(\frac{\delta x}{y}\Big)\Big\vert,\\
			\leq & \dfrac{\ln{\varphi}}{\vert \ln(\frac{\delta x}{y})\vert}\Big\vert\ln\Big(\frac{\delta x}{y}\Big)\Big\vert,\\
			= & \ln\varphi.
		\end{aligned}
	\end{equation}
	By rearranging \eqref{eq:main_vert_ineq}, the we obtain
	\begin{equation}
		\begin{aligned}
			\ln\left(\frac{y}{\delta x\varphi}\right)
			& \leq
			\dfrac{1}{\text{LID}(\mathbf{b})}\ln\Big(\frac{F_{\mathbf{b}}(y)}{F_{\mathbf{b}}(\delta x)}\Big)
			&& \leq
			\ln\left(\frac{y\varphi}{\delta x}\right)\\
			\ln\Big(\frac{y}{\delta x}\cdot\frac{y-\delta x\eta}{y}\Big)
			& \leq	\dfrac{1}{\text{LID}(\mathbf{b})}\ln\Big(\frac{F_{\mathbf{b}}(y)}{F_{\mathbf{b}}(\delta x)}\Big)
			&& \leq
			\ln\Big(\frac{y}{\delta x}\cdot\frac{y+\delta x\eta}{y}\Big)\\
			\ln\Big(\frac{y}{\delta x}-\eta\Big)
			& \leq
			\dfrac{1}{\text{LID}(\mathbf{b})}\ln\Big(\frac{F_{\mathbf{b}}(y)}{F_{\mathbf{b}}(\delta x)}\Big)
			&& \leq
			\ln\Big(\frac{y}{\delta x}+\eta\Big).
		\end{aligned}
	\end{equation}
	
	Upon rearranging, we obtain
	\begin{equation}\label{eq:inequalities_lid_2}
		\dfrac{\ln\big(F_{\mathbf{b}}(y)\big/F_{\mathbf{b}}(\delta x)\big)}{\ln\Big(\frac{y}{\delta x}+\eta\Big)}\leq
		\text{LID}(\mathbf{b})\leq
		\dfrac{\ln\big(F_{\mathbf{b}}(y)\big/F_{\mathbf{b}}(\delta x)\big)}{\ln\Big(\frac{y}{\delta x}-\eta\Big)}.
	\end{equation}
\end{proof}

Theorem \ref{thm:main_theorem} derives a lower-bound and an upper-bound for $\text{LID}(\mathbf{b})$ using the cumulative probability values (of the distribution of distances from $\mathbf{b}$) of the remaining two data points and the distances to them. Since the maximum value $F_{\mathbf{b}}(\delta x)$ can have is one, we observe a strong association between the size of perturbations $\delta$ and the LID value of $\mathbf{b}$. This theoretical result justifies the use of LID as an indicator to detect adversarial perturbations in prior works (Section \ref{sec:lit_review}).

\subsection{Additional Bounds based on the Perturbation Directions}
\begin{figure}[h]
	\centering
	\begin{subfigure}{0.4\textwidth} 
		\tikzset{every picture/.style={line width=0.75pt}} 
		\begin{tikzpicture}[x=0.75pt,y=0.75pt,yscale=-1,xscale=1]
			\draw  [fill={rgb, 255:red, 155; green, 155; blue, 155 }  ,fill opacity=0.36 ][dash pattern={on 0.84pt off 2.51pt}] (523.75,141.33) .. controls (517.35,149.64) and (507.3,155) .. (496,155) .. controls (476.67,155) and (461,139.33) .. (461,120) .. controls (461,100.67) and (476.67,85) .. (496,85) .. controls (506.21,85) and (515.41,89.38) .. (521.8,96.35) -- (496,120) -- cycle ;
			\draw    (496,120) -- (597.33,120.33) ;
			\draw [shift={(597.33,120.33)}, rotate = 0.19] [color={rgb, 255:red, 0; green, 0; blue, 0 }  ][fill={rgb, 255:red, 0; green, 0; blue, 0 }  ][line width=0.75]      (0, 0) circle [x radius= 3.35, y radius= 3.35]   ;
			\draw    (442.44,81.74) -- (496,120) ;
			\draw [shift={(496,120)}, rotate = 35.54] [color={rgb, 255:red, 0; green, 0; blue, 0 }  ][fill={rgb, 255:red, 0; green, 0; blue, 0 }  ][line width=0.75]      (0, 0) circle [x radius= 3.35, y radius= 3.35]   ;
			\draw [shift={(440,80)}, rotate = 35.54] [fill={rgb, 255:red, 0; green, 0; blue, 0 }  ][line width=0.08]  [draw opacity=0] (8.93,-4.29) -- (0,0) -- (8.93,4.29) -- cycle    ;
			\draw  [draw opacity=0] (479.38,107.16) .. controls (481.16,101.53) and (485.65,97.2) .. (491.94,95.97) .. controls (502.33,93.93) and (513.66,101.09) .. (517.26,111.98) .. controls (518.09,114.48) and (518.43,116.97) .. (518.35,119.33) -- (498.45,115.67) -- cycle ; \draw   (479.38,107.16) .. controls (481.16,101.53) and (485.65,97.2) .. (491.94,95.97) .. controls (502.33,93.93) and (513.66,101.09) .. (517.26,111.98) .. controls (518.09,114.48) and (518.43,116.97) .. (518.35,119.33) ;
			\draw    (496,90) -- (496,150) ;
			
			\draw (477,112.4) node [anchor=north west][inner sep=0.75pt]    {$\mathbf{a}$};
			\draw (600,110.4) node [anchor=north west][inner sep=0.75pt]    {$\mathbf{c}$};
			\draw (457,62.4) node [anchor=north west][inner sep=0.75pt]    {$\mathbf{b}$};
			\draw (557,122.4) node [anchor=north west][inner sep=0.75pt]    {$x$};
			\draw (523.8,99.75) node [anchor=north west][inner sep=0.75pt]    {$\theta $};
			\draw (431,102.4) node [anchor=north west][inner sep=0.75pt]    {$\delta x$};
		\end{tikzpicture}
	\end{subfigure}
	\begin{subfigure}{0.5\textwidth} 
		\begin{tikzpicture}[x=0.75pt,y=0.75pt,yscale=-1,xscale=1]
			\draw    (70,40) -- (280,40) ;
			\draw  [color={rgb, 255:red, 74; green, 144; blue, 226 }  ,draw opacity=1 ] (132.88,39.88) .. controls (132.88,35.94) and (136.06,32.75) .. (140,32.75) .. controls (143.94,32.75) and (147.13,35.94) .. (147.13,39.88) .. controls (147.13,43.81) and (143.94,47) .. (140,47) .. controls (136.06,47) and (132.88,43.81) .. (132.88,39.88) -- cycle ;
			\draw  [color={rgb, 255:red, 74; green, 144; blue, 226 }  ,draw opacity=1 ] (273,39.88) .. controls (273,35.94) and (276.19,32.75) .. (280.13,32.75) .. controls (284.06,32.75) and (287.25,35.94) .. (287.25,39.88) .. controls (287.25,43.81) and (284.06,47) .. (280.13,47) .. controls (276.19,47) and (273,43.81) .. (273,39.88) -- cycle ;
			\draw    (210,30) -- (210,50) ;
			\draw    (70,30) -- (70,50) ;
			\draw    (140,30) -- (140,50) ;
			\draw    (280,30) -- (280,50) ;
			\draw    (70,40) -- (140,39.88) (79.99,35.98) -- (80.01,43.98)(89.99,35.96) -- (90.01,43.96)(99.99,35.95) -- (100.01,43.95)(109.99,35.93) -- (110.01,43.93)(119.99,35.91) -- (120.01,43.91)(129.99,35.89) -- (130.01,43.89)(139.99,35.88) -- (140.01,43.87) ;
			\draw    (210,40) -- (280,40) (220,36) -- (220,44)(230,36) -- (230,44)(240,36) -- (240,44)(250,36) -- (250,44)(260,36) -- (260,44)(270,36) -- (270,44) ;
			
			\draw (154.5,13.4) node [anchor=north west][inner sep=0.75pt]    {$\text{LID}(\mathbf{b})$};
			\draw (203,12) node [anchor=north west][inner sep=0.75pt]   [align=left] {upper-bound};
			\draw (64,12) node [anchor=north west][inner sep=0.75pt]   [align=left] {lower-bound};
		\end{tikzpicture}
	\end{subfigure}
	\caption{The figure on the left shows the valid range of $\theta$ (grey) when $\mathbf{b}$ is perturbed away from $\mathbf{c}$. The figure on the right shows the least lower and upper bounds when $\mathbf{b}$ is perturbed away from $\mathbf{c}$.}
	\label{fig:perturbation_away_angle}
\end{figure}
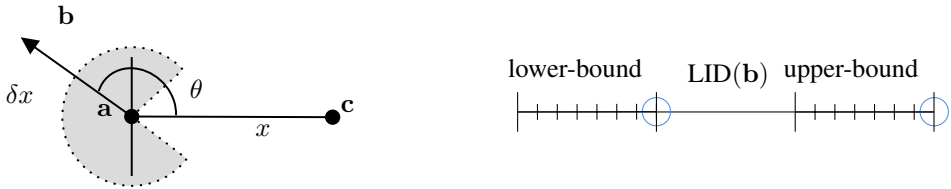

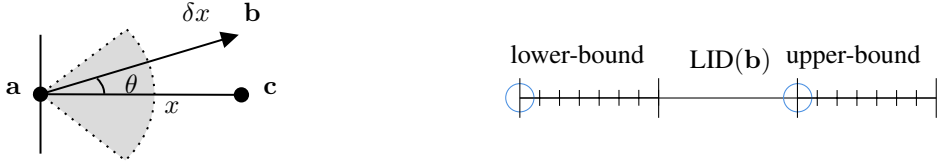
\begin{figure}[t]
	\centering
	\begin{subfigure}{0.4\textwidth} 
		\tikzset{every picture/.style={line width=0.75pt}} 
		\begin{tikzpicture}[x=0.75pt,y=0.75pt,yscale=-1,xscale=1]
			\draw  [fill={rgb, 255:red, 155; green, 155; blue, 155 }  ,fill opacity=0.36 ][dash pattern={on 0.84pt off 2.51pt}] (133.49,47.31) .. controls (142.2,56.08) and (147.47,67.5) .. (147.47,80) .. controls (147.47,92.86) and (141.9,104.58) .. (132.73,113.44) -- (90,80) -- cycle ;
			\draw    (90,80) -- (191.33,80.33) ;
			\draw [shift={(191.33,80.33)}, rotate = 0.19] [color={rgb, 255:red, 0; green, 0; blue, 0 }  ][fill={rgb, 255:red, 0; green, 0; blue, 0 }  ][line width=0.75]      (0, 0) circle [x radius= 3.35, y radius= 3.35]   ;
			\draw    (187.13,50.86) -- (90,80) ;
			\draw [shift={(90,80)}, rotate = 163.3] [color={rgb, 255:red, 0; green, 0; blue, 0 }  ][fill={rgb, 255:red, 0; green, 0; blue, 0 }  ][line width=0.75]      (0, 0) circle [x radius= 3.35, y radius= 3.35]   ;
			\draw [shift={(190,50)}, rotate = 163.3] [fill={rgb, 255:red, 0; green, 0; blue, 0 }  ][line width=0.08]  [draw opacity=0] (8.93,-4.29) -- (0,0) -- (8.93,4.29) -- cycle    ;
			\draw  [draw opacity=0] (118.52,72.06) .. controls (120.04,73.58) and (121.14,75.26) .. (121.71,77.02) .. controls (122.01,77.96) and (122.14,78.88) .. (122.12,79.77) -- (101.91,76.81) -- cycle ; \draw   (118.52,72.06) .. controls (120.04,73.58) and (121.14,75.26) .. (121.71,77.02) .. controls (122.01,77.96) and (122.14,78.88) .. (122.12,79.77) ;
			\draw    (90,50) -- (90,110) ;
			
			\draw (71,72.4) node [anchor=north west][inner sep=0.75pt]    {$\mathbf{a}$};
			\draw (201,72.4) node [anchor=north west][inner sep=0.75pt]    {$\mathbf{c}$};
			\draw (191,32.4) node [anchor=north west][inner sep=0.75pt]    {$\mathbf{b}$};
			\draw (151,82.4) node [anchor=north west][inner sep=0.75pt]    {$x$};
			\draw (131,68.4) node [anchor=north west][inner sep=0.75pt]    {$\theta $};
			\draw (160,32.4) node [anchor=north west][inner sep=0.75pt]    {$\delta x$};
		\end{tikzpicture}
	\end{subfigure}
	\begin{subfigure}{0.5\columnwidth} 
		\begin{tikzpicture}[x=0.75pt,y=0.75pt,yscale=-1,xscale=1]
			
			\draw    (70,40) -- (280,40) ;
			\draw  [color={rgb, 255:red, 74; green, 144; blue, 226 }  ,draw opacity=1 ] (62.88,39.88) .. controls (62.88,35.94) and (66.06,32.75) .. (70,32.75) .. controls (73.94,32.75) and (77.13,35.94) .. (77.13,39.88) .. controls (77.13,43.81) and (73.94,47) .. (70,47) .. controls (66.06,47) and (62.88,43.81) .. (62.88,39.88) -- cycle ;
			\draw  [color={rgb, 255:red, 74; green, 144; blue, 226 }  ,draw opacity=1 ] (203,39.88) .. controls (203,35.94) and (206.19,32.75) .. (210.13,32.75) .. controls (214.06,32.75) and (217.25,35.94) .. (217.25,39.88) .. controls (217.25,43.81) and (214.06,47) .. (210.13,47) .. controls (206.19,47) and (203,43.81) .. (203,39.88) -- cycle ;
			\draw    (210,30) -- (210,50) ;
			\draw    (70,30) -- (70,50) ;
			\draw    (140,30) -- (140,50) ;
			\draw    (280,30) -- (280,50) ;
			\draw    (70,40) -- (140,39.88) (79.99,35.98) -- (80.01,43.98)(89.99,35.96) -- (90.01,43.96)(99.99,35.95) -- (100.01,43.95)(109.99,35.93) -- (110.01,43.93)(119.99,35.91) -- (120.01,43.91)(129.99,35.89) -- (130.01,43.89)(139.99,35.88) -- (140.01,43.87) ;
			\draw    (210,40) -- (280,40) (220,36) -- (220,44)(230,36) -- (230,44)(240,36) -- (240,44)(250,36) -- (250,44)(260,36) -- (260,44)(270,36) -- (270,44) ;
			
			\draw (154.5,13.4) node [anchor=north west][inner sep=0.75pt]    {$\text{LID}(\mathbf{b})$};
			\draw (203,12) node [anchor=north west][inner sep=0.75pt]   [align=left] {upper-bound};
			\draw (64,12) node [anchor=north west][inner sep=0.75pt]   [align=left] {lower-bound};
		\end{tikzpicture}
	\end{subfigure}
	\caption{The figure on the left shows the valid range of $\theta$ (grey) when $\mathbf{b}$ is perturbed towards $\mathbf{c}$. The figure on the right shows the least lower and upper bounds when $\mathbf{b}$ is perturbed towards $\mathbf{c}$.}
	\label{fig:perturbation_towards_angle}
\end{figure}

We now consider the relationship between LID and the direction of perturbations. The bounds in Theorem \ref{thm:main_theorem} hold for $\mathbf{b}$ irrespective of the perturbation directions. In Lemma \ref{lem:1}, we present the conditions $\delta$ and $\theta$ must satisfy given the two possible directions of perturbation. In the following theorem, we fix $\delta$, and consider the possible angles $\theta$ can take to obtain the greatest lower-bound and greatest lower-bound of \eqref{eq:lid_y_inequality} when $\mathbf{b}$ is perturbed away from $\mathbf{c}$ (relative to $\mathbf{a}$) (refer Figure \ref{fig:perturbation_away_angle}). Conversely, we obtain the least lower-bound and least upper-bound of \eqref{eq:lid_y_inequality} when $\mathbf{b}$ is perturbed towards $\mathbf{c}$ (refer Figure \ref{fig:perturbation_towards_angle}).

\begin{theorem}\label{thm:bounds}
	Denote the distances $x=\Vert \mathbf{a}-\mathbf{c}\Vert$, $\Vert \mathbf{b}-\mathbf{c}\Vert$ as $y$ and $\delta x=\Vert \mathbf{a}-\mathbf{b}\Vert$ for an appropriate choice of $0<\delta<y/x$. We denote the distributional ranks of $\mathbf{a}$ and $\mathbf{c}$ as $p_{\mathbf{a}}=F_{\mathbf{b}}(\delta x)$ and $p_{\mathbf{c}}=F_{\mathbf{b}}(y)$, and the distribution of distances from $\mathbf{b}\in\mathcal{S}$ as $\mathcal{D}_{\mathbf{b}}$. Assume $\text{LID}(\mathbf{b})$ exists and is positive. Assume $\eta$ is a sufficiently small real value that satisfies the constraints outlined in Theorem \ref{thm:main_theorem}. Additionally, it satisfies the constraint $1-\frac{1}{\delta}<\eta<\frac{1}{\delta}-1$. 
	
	By considering the boundary values of $\cos\theta$ as per Lemma \ref{lem:1}, when $\mathbf{b}$ is perturbed farther from $\mathbf{c}$ (i.e., $y>x$), the greatest lower and upper bounds of \eqref{eq:inequalities_lid_2} are as follows:
	\begin{equation}
		\begin{aligned}\label{eq:lid_y_farther}
			\dfrac{\ln\big(F_{\mathbf{b}}(y)\big/F_{\mathbf{b}}(\delta x)\big)}{\ln(\frac{y}{\delta x}+\eta)}<  
			\dfrac{\ln\big(F_{\mathbf{b}}(y)\big/F_{\mathbf{b}}(\delta x)\big)}{\ln(\frac{1}{\delta}+\eta)}< 
			\text{LID}(\mathbf{b})\leq\\                                                               
			\dfrac{\ln\big(F_{\mathbf{b}}(y)\big/F_{\mathbf{b}}(\delta x)\big)}{\ln(\frac{y}{\delta x}-\eta)}<  
			\dfrac{\ln\big(F_{\mathbf{b}}(y)\big/F_{\mathbf{b}}(\delta x)\big)}{\ln(\frac{1}{\delta}-\eta)}.    
		\end{aligned}
	\end{equation}
	Conversely, when $\mathbf{b}$ is perturbed towards $\mathbf{c}$ (i.e., $y<x$), the least lower and upper bounds of \eqref{eq:inequalities_lid_2} are as follows:
	\begin{equation}
		\begin{aligned}\label{eq:lid_y_closer}
			\dfrac{\ln\big(F_{\mathbf{b}}(y)\big/F_{\mathbf{b}}(\delta x)\big)}{\ln(\frac{1}{\delta}+\eta)}<
			\dfrac{\ln\big(F_{\mathbf{b}}(y)\big/F_{\mathbf{b}}(\delta x)\big)}{\ln(\frac{y}{\delta x}+\eta)}\leq 
			\text{LID}(\mathbf{b})<\\
			\dfrac{\ln\big(F_{\mathbf{b}}(y)\big/F_{\mathbf{b}}(\delta x)\big)}{\ln(\frac{1}{\delta}-\eta)}<
			\dfrac{\ln\big(F_{\mathbf{b}}(y)\big/F_{\mathbf{b}}(\delta x)\big)}{\ln(\frac{y}{\delta x}-\eta)}.    
		\end{aligned}
	\end{equation}
\end{theorem}

\begin{proof}
	We consider the two possible directions of perturbation separately.\\
	\textbf{Case 1: $\mathbf{b}$ perturbed farther from $\mathbf{c}$}\\
	First, we consider the case where $\mathbf{b}$ is perturbed farther from $\mathbf{c}$ (i.e., $y>x$). Using \eqref{eq:y_distance} on the denominator of the left inequality of \eqref{eq:inequalities_lid_2}, we obtain
	\begin{equation}
		\begin{aligned}
			\ln\Big(\frac{y}{\delta x}+\eta\Big)
			& =\ln\Big(                    
			\frac{\sqrt{\delta^2 x^2+x^2-2\delta x^2\cos\theta}}{\delta x}+\eta\Big),\\
			& =\ln\Big(                    
			\sqrt{1+\frac{1}{\delta^2}-\frac{2\cos\theta}{\delta}}+\eta\Big).
		\end{aligned}
	\end{equation}
	We then consider the upper boundary value of $\cos\theta$ as shown in \eqref{eq:xyz_distances_farther} (i.e., $\cos\theta<\delta/2$), where we get
	\begin{equation}
		\begin{aligned}
			\ln\Big(                    
			\sqrt{1+\frac{1}{\delta^2}-\frac{2\cos\theta}{\delta}}+\eta\Big)
			& >\ln\Big(                    
			\sqrt{1+\frac{1}{\delta^2}-1}+\eta\Big)\\
			& =\ln\left(\frac{1}{\delta}+\eta\right). 
		\end{aligned}
	\end{equation}
	Thus, for all feasible values of $\theta$, we observe that $\ln(\frac{y}{\delta x}+\eta)>\ln(\frac{1}{\delta}+\eta)$. Therefore, by replacing $\ln(\frac{y}{\delta x}+\eta)$ in \eqref{eq:inequalities_lid_2} with $\ln(\frac{1}{\delta}+\eta)$, we obtain the following greatest lower bound for $\text{LID}(\mathbf{b})$
	\begin{equation}\label{eq:inequalities_lid_lb}
		\dfrac{\ln\big(F_{\mathbf{b}}(y)\big/F_{\mathbf{b}}(\delta x)\big)}{\ln\Big(\frac{y}{\delta x}+\eta\Big)}<
		\dfrac{\ln\big(F_{\mathbf{b}}(y)\big/F_{\mathbf{b}}(\delta x)\big)}{\ln(\frac{1}{\delta}+\eta)}<
		\text{LID}(\mathbf{b}).
	\end{equation}
	The greatest upper bound of \eqref{eq:lid_y_farther} can be obtained by replacing $\ln\big(\frac{y}{\delta x}-\eta\big)$ with $\ln(\frac{1}{\delta}-\eta)$ following an identical approach. 
	\begin{equation}\label{eq:inequalities_lid_ub}
		\text{LID}(\mathbf{b})\leq
		\dfrac{\ln\big(F_{\mathbf{b}}(y)\big/F_{\mathbf{b}}(\delta x)\big)}{\ln(\frac{y}{\delta x}-\eta)}<
		\dfrac{\ln\big(F_{\mathbf{b}}(y)\big/F_{\mathbf{b}}(\delta x)\big)}{\ln(\frac{1}{\delta}-\eta)}.
	\end{equation}
	
	Thus, from \eqref{eq:inequalities_lid_lb} and \eqref{eq:inequalities_lid_ub}, we see that when $\mathbf{b}$ is perturbed farther from the reference sample $\mathbf{c}$, the greatest lower and upper bounds of \eqref{eq:inequalities_lid_2} are as follows:
	\begin{equation}\label{eq:greatest_lower_and_upper}
		\begin{aligned}
			\dfrac{\ln\big(F_{\mathbf{b}}(y)\big/F_{\mathbf{b}}(\delta x)\big)}{\ln(\frac{y}{\delta x}+\eta)}<  
			\dfrac{\ln\big(F_{\mathbf{b}}(y)\big/F_{\mathbf{b}}(\delta x)\big)}{\ln(\frac{1}{\delta}+\eta)}<
			\text{LID}(\mathbf{b})\leq\\                                                               
			\dfrac{\ln\big(F_{\mathbf{b}}(y)\big/F_{\mathbf{b}}(\delta x)\big)}{\ln(\frac{y}{\delta x}-\eta)}<  
			\dfrac{\ln\big(F_{\mathbf{b}}(y)\big/F_{\mathbf{b}}(\delta x)\big)}{\ln(\frac{1}{\delta}-\eta)}.    
		\end{aligned}
	\end{equation}
	\textbf{Case 2: $\mathbf{b}$ perturbed closer to $\mathbf{c}$}\\
	We now consider the case where $\mathbf{b}$ is perturbed towards $\mathbf{c}$ (i.e., $y<x$). Take the lower boundary value of $\cos\theta$ as shown in \eqref{eq:xyz_distances_towards} (i.e., $\cos\theta>\delta/2$) in order to obtain the greatest possible denominator value as follows:
	\begin{equation}
		\begin{aligned}
			\ln\Big(\frac{y}{\delta x}+\eta\Big)
			& =\ln\Big(                    
			\sqrt{1+\frac{1}{\delta^2}-\frac{2\cos\theta}{\delta}}+\eta\Big)\\
			& <\ln\Big(                    
			\sqrt{1+\frac{1}{\delta^2}-1}+\eta\Big)\\
			& =\ln\left(\frac{1}{\delta}+\eta\right). 
		\end{aligned}
	\end{equation}
	Thus, for all feasible values of $\theta$, we have $\ln(\frac{y}{\delta x}+\eta)<\ln(\frac{1}{\delta}+\eta)$. By replacing $\ln(\frac{y}{\delta x}+\eta)$ in \eqref{eq:lid_y_closer} with $\ln(\frac{1}{\delta}+\eta)$, we obtain the following least lower bound for $\text{LID}(\mathbf{b})$
	\begin{equation}\label{eq:inequalities_lid_lb_towards}
		\dfrac{\ln\big(F_{\mathbf{b}}(y)\big/F_{\mathbf{b}}(\delta x)\big)}{\ln(\frac{1}{\delta}+\eta)}<
		\dfrac{\ln\big(F_{\mathbf{b}}(y)\big/F_{\mathbf{b}}(\delta x)\big)}{\ln\Big(\frac{y}{\delta x}+\eta\Big)}\leq
		\text{LID}(\mathbf{b}).
	\end{equation}
	The least upper bound of \eqref{eq:lid_y_farther} can be obtained by replacing $\ln\big(\frac{y}{\delta x}-\eta\big)$ with $\ln(\frac{1}{\delta}-\eta)$ following an identical approach. 
	\begin{equation}\label{eq:inequalities_lid_ub_towards}
		\text{LID}(\mathbf{b})<
		\dfrac{\ln\big(F_{\mathbf{b}}(y)\big/F_{\mathbf{b}}(\delta x)\big)}{\ln(\frac{1}{\delta}-\eta)}<
		\dfrac{\ln\big(F_{\mathbf{b}}(y)\big/F_{\mathbf{b}}(\delta x)\big)}{\ln(\frac{y}{\delta x}-\eta)}.
	\end{equation}
	
	Therefore, from \eqref{eq:inequalities_lid_lb_towards} and \eqref{eq:inequalities_lid_ub_towards}, we observe the least lower and upper bounds of \eqref{eq:inequalities_lid_2} when $\mathbf{b}$ is perturbed towards from the reference sample $\mathbf{c}$.
	\begin{equation}\label{eq:least_lower_and_upper}
		\begin{aligned}
			\dfrac{\ln\big(F_{\mathbf{b}}(y)\big/F_{\mathbf{b}}(\delta x)\big)}{\ln(\frac{1}{\delta}+\eta)}<
			\dfrac{\ln\big(F_{\mathbf{b}}(y)\big/F_{\mathbf{b}}(\delta x)\big)}{\ln(\frac{y}{\delta x}+\eta)}\leq  
			\text{LID}(\mathbf{b})<\\  
			\dfrac{\ln\big(F_{\mathbf{b}}(y)\big/F_{\mathbf{b}}(\delta x)\big)}{\ln(\frac{1}{\delta}-\eta)}\leq                                                
			\dfrac{\ln\big(F_{\mathbf{b}}(y)\big/F_{\mathbf{b}}(\delta x)\big)}{\ln(\frac{y}{\delta x}-\eta)}.    
		\end{aligned}
	\end{equation}
\end{proof}

We note that in practice, as evidenced by the experimental results in the section that follows, attack algorithms tend to perturb data away from the domain of the benign data. Therefore, the most important result of Theorem \ref{thm:bounds} is the greatest lower-bound in \eqref{eq:greatest_lower_and_upper}. This greatest-lower bound further supports the strong association seen between the size of perturbations $\delta$ and the LID value of $\mathbf{b}$.

\section{Experimental Validation and Discussion}\label{sec:results}
In this section, we empirically evaluate the lower-bound and upper-bound derived in Theorem \ref{thm:main_theorem} for the LID value of the perturbed data point $\mathbf{b}$. Our code is available at \url{https://github.com/sandamal/adversarial_lid_bounds}.

To approximate the cumulative distributions of the distance variable (i.e., $F_{b}$), we first use kernel density estimation with a Gaussian kernel to obtain the probability density function of $\mathcal{D}_b$ \citep{silverman2018density}. Subsequently, the integral is calculated to obtain the cumulative distribution. We empirically evaluate the theoretical bounds on two real-world datasets: MNIST \citep{lecun2010mnist} and CIFAR-10 \citep{Krizhevsky09learningmultiple}. The datasets contain 10,000 samples with 782 and 3,072 features respectively. It should be noted that, similar to the work of \citet{amsaleg2020high}, the bounds hold asymptotically as the number of data samples $n$ tends to infinity. However, due to the high computational complexity of kernel density estimation, we perform our experiments on a randomly sampled subset of data points with $n\in\{1000,5000\}$.

\begin{figure}[t]
	\centering
	\begin{subfigure}{0.4\textwidth} 
		\includegraphics[width=\textwidth]{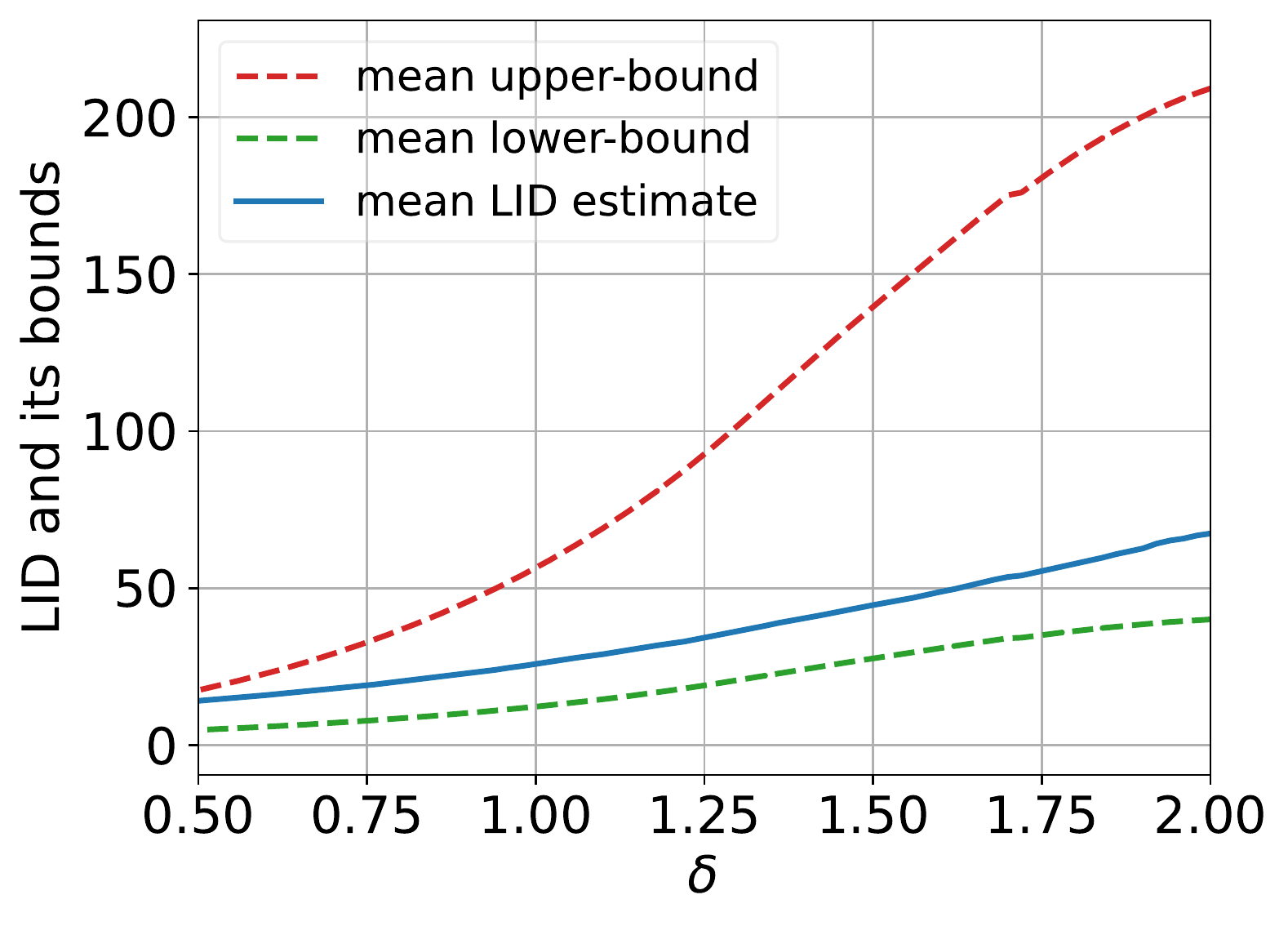}
		\caption{MNIST: $k=100$, $n_q=50$, $n=1000$.} 
		\label{fig:mnist_100}
	\end{subfigure}
	\begin{subfigure}{0.4\textwidth} 
		\includegraphics[width=\textwidth]{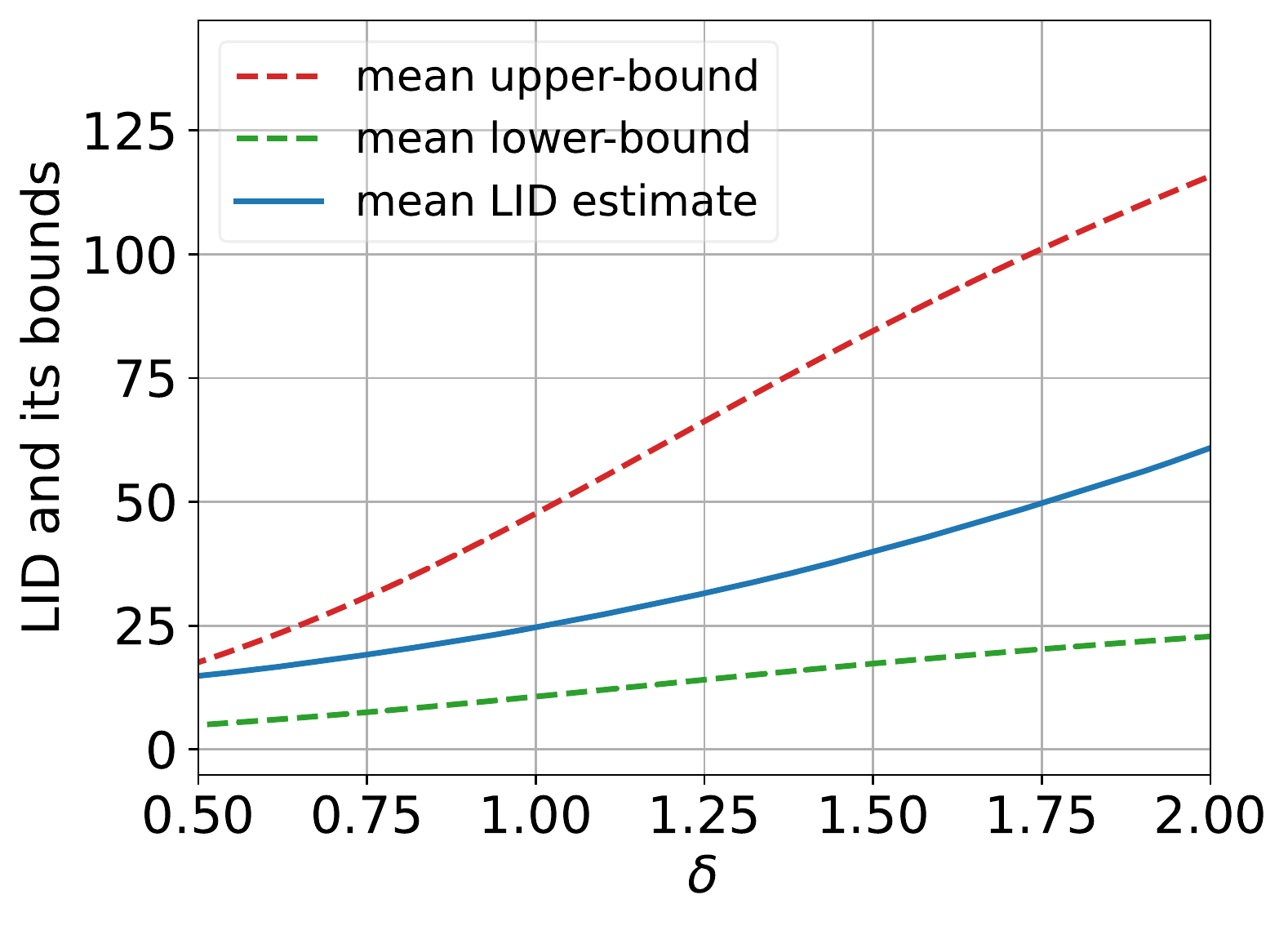}
		\caption{CIFAR-10: $k=100$, $n_q=50$, $n=1000$.} 
		\label{fig:cifar_100}
	\end{subfigure}
	\begin{subfigure}{0.4\textwidth} 
		\includegraphics[width=\textwidth]{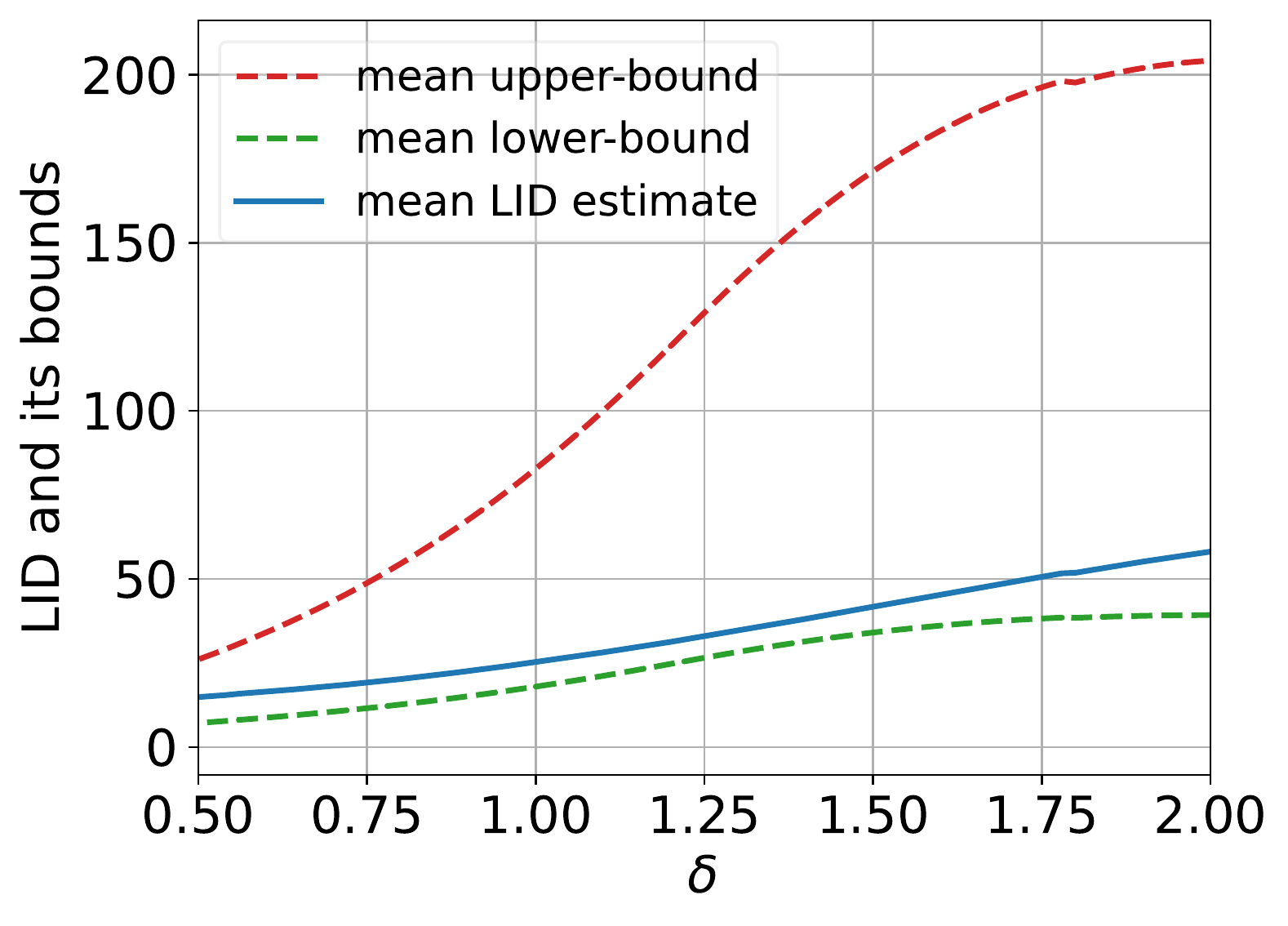}
		\caption{MNIST: $k=1000$, $n_q=50$, $n=5000$.} 
		\label{fig:mnist_1000}
	\end{subfigure}
	\begin{subfigure}{0.4\textwidth} 
		\includegraphics[width=\textwidth]{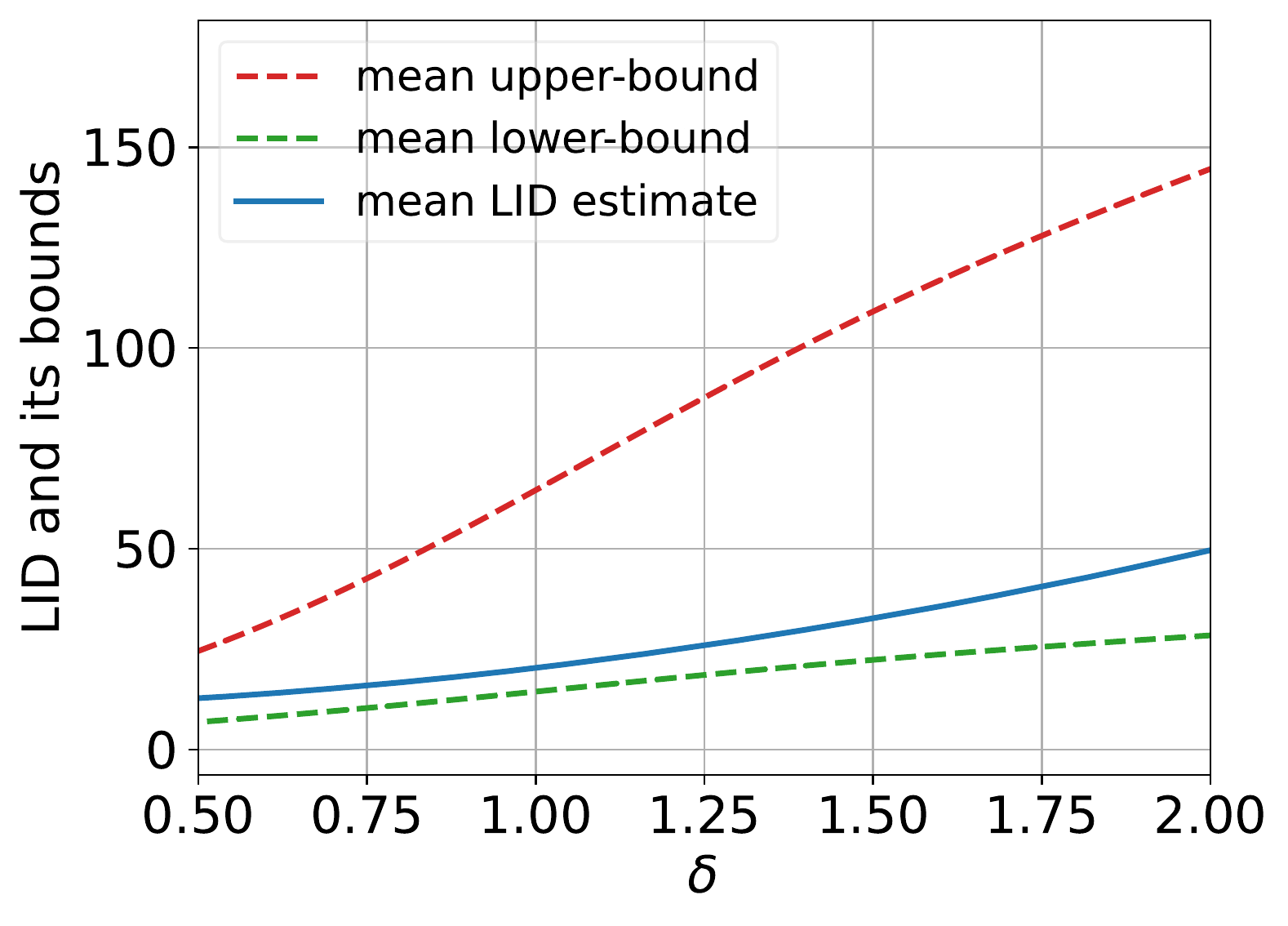}
		\caption{CIFAR-10: $k=1000$, $n_q=50$, $n=5000$.} 
		\label{fig:cifar_1000}
	\end{subfigure}
	\caption{The LID estimates and theoretical bounds on MNIST and CIFAR-10 when data points are poisoned using the gradient-based attack.}
	\label{fig:lid_bounds}
\end{figure}
We randomly select $n_{q}=50$ query points from each dataset (i.e., $\mathbf{a}$) and select the perturbation direction using a gradient based attack \citep{43405}. For the LID estimations \eqref{eq:lid_estimation} we select the neighborhood size $k$ from $\{100,1000\}$. For each query point $\mathbf{a}$, we select the data point with the rank $k/2$ from it as the reference point (i.e., $\mathbf{c}$). We then obtain the perturbed points (i.e., $\mathbf{b}$) by perturbing the query points in the direction that maximizes the training error (according to the attack algorithm). 

Figure \ref{fig:lid_bounds} shows the empirical evaluations of the theoretical bounds in Theorem \ref{thm:main_theorem}. It shows the averaged bounds for all feasible $\eta$ values as well as the estimated LID values against the perturbation $\delta$. As depicted in the figure, when $\delta$ increases, the lower-bound, LID estimate and the upper-bound increase. The most important indicator here is the lower-bound that increases with $\delta$ as it supports the intuition that LID values of adversarial samples tend to be larger. 

It should be noted that when $\delta$ is close to zero, the bounds in \eqref{eq:lid_y_inequality} exhibit some numerical instability as $F_{\mathbf{b}}(y)\big/F_{\mathbf{b}}(\delta x$), $\frac{y}{\delta x}+\eta$ and $\frac{y}{\delta x}-\eta$ tend to infinity. However, in practice, it is expected for adversarial perturbations to be sufficiently large in order to influence the learner's prediction capabilities. Therefore it is reasonable to expect $\delta$ to not be close to zero.

We observe similar numerical instabilities for \eqref{eq:greatest_lower_and_upper} and \eqref{eq:least_lower_and_upper} of Theorem \ref{thm:bounds} as well. Furthermore, when $\delta$ tends to one (i.e., the theoretical maximum value for which Theorem \ref{thm:bounds} theoretically holds), since $\eta$ is a small positive value, we observe that both lower-bounds and upper-bounds in \eqref{eq:greatest_lower_and_upper} and \eqref{eq:least_lower_and_upper} tend to infinity. Although we have considered $\delta$ values up to 2.5 (which makes the maximum perturbation $2.5x$), in practice, we observe that adversarial attacks tend to have bounded perturbations (to avoid detection).

As for the directions of perturbation, we observed that, on average, only $5.56\%$ and $0.65\%$ of the total perturbations were made in the direction towards the benign reference point $\mathbf{c}$ for MNIST and CIFAR-10 respectively. Intuitively, we expect a gradient-based attack that perturbs data in the direction that maximizes the training error to perturb data points away from the benign data domain. This finding, while preliminary, supports this intuition.

\section{Conclusions}\label{sec:conclusions}
The purpose of this paper was to theoretically determine the relationship between the magnitude of the perturbations and the LID value of the resulting perturbed data point. To that end, we derived a lower-bound and an upper-bound for the LID value of a perturbed data point based on the cumulative probability (of its distance distribution) values to its original position and a third benign reference point. The bounds, in particular the lower-bound, showed that there is a positive association between the magnitude of the perturbation $\delta$ and the LID value of the perturbed point. We then empirically evaluated the effectiveness of these bounds on two reference data sets using a gradient-based poisoning attack. Furthermore, by considering the two possible directions of perturbation, we obtained the least lower-bound and upper-bound when the direction is towards the benign reference point and the greatest lower-bound and upper-bound when the direction is away from it. These findings support the use of LID as a tool to detect adversarial samples in prior state of the art works.


\bibliographystyle{unsrtnat}
\bibliography{references.bib}  






\end{document}